\documentclass{article}
\usepackage[final]{score_workshop}

\usepackage{hyperref}  
\hypersetup{
	colorlinks=true,
	linkcolor=black,
	citecolor=black,
	filecolor=black,
	urlcolor=black,
} 
\PassOptionsToPackage{round}{natbib}

\usepackage[utf8]{inputenc} 
\usepackage[T1]{fontenc}    
\usepackage{url}            
\usepackage{booktabs}       
\usepackage{amsfonts}       
\usepackage{nicefrac}       
\usepackage{microtype}      
\usepackage{xcolor}         

\usepackage{bm,mathtools,dsfont,mathrsfs}
\usepackage{IEEEtrantools}
\usepackage[shortlabels, inline]{enumitem}
\usepackage[title]{appendix}
\usepackage{graphicx}
\usepackage{tikz}
\usepackage{amsthm}

\newtheorem{proposition}{Proposition}

\DeclareMathOperator{\R}{\mathbb{R}}

\DeclareMathOperator{\D}{\mathbf{D}}
\DeclareMathOperator{\simplex}{\mathcal{S}}
\DeclareMathOperator{\Hyv}{\mathcal{H}}

\DeclareMathOperator*{\argmin}{arg\,min}

\newcommand{\bigoh}[1]{\mathcal{O}(#1)}

\title{Locking and Quacking: Stacking Bayesian model predictions by log-pooling and superposition}
\author{
  Yuling Yao  \\
 Flatiron Institute
  \And
  Luiz Max Carvalho \\
   { \hspace{1.9cm} Getulio Vargas}
  \And
    Diego Mesquita \\
    {\hspace{-1.9cm} Foundation}
    \And
   Yann McLatchie  \\
   Aalto Univerity
}
\date{}

\begin{document}
\maketitle
\thispagestyle{empty}
\begin{abstract}
    Combining predictions from different models is a central problem in Bayesian inference and machine learning more broadly.
    Currently, these predictive distributions are almost exclusively combined using linear mixtures such as Bayesian model averaging, Bayesian stacking, and mixture of experts.
    Such linear mixtures impose idiosyncrasies that might be undesirable for some applications, such as multi-modality.
    While there exist alternative strategies (e.g. geometric bridge or superposition), optimising their parameters usually involves computing an intractable normalising constant repeatedly.
    We present two novel Bayesian model combination tools. These are generalisations of model stacking, but combine posterior densities 
    by log-linear pooling (\emph{locking}) and quantum superposition (\emph{quacking}). 
    To optimise model weights while avoiding the burden of normalising constants, we investigate the  Hyv\"arinen score of the combined posterior predictions.
    We demonstrate locking with an illustrative example, and discuss its practical application with importance sampling.
\end{abstract}
\section{Introduction}
\label{sec:introduction}

A general challenge in statistics is prediction in the presence of multiple candidate models or learning algorithms: we are interested in some outcome $y$ on a measurable space $\mathcal{Y}\subseteq \R^d$; we fit different models to the data, or the same model on different parts of the data set leading to a set of predictive distributions, $\{\pi_1 (y), \ldots, \pi_K (y) \}$, where each $\pi_k (y)$ is  a (conditional\footnote{The dependence on covariates $\theta$ is suppressed for brevity.}) probabilistic density such that $\int_{\mathcal{Y}} \pi_{k} (y) \,\mathrm{d}y =1.$
There are three subjective decisions to make in such a workflow:
\begin{enumerate*}[label=(\arabic*), ref=\arabic*, nosep]
\item the choice of individual models to combine;
\item the  prior weight assigned to each model;
\item the form in which individual sampling models are combined in the predictive sampling distribution.
\end{enumerate*}
We are primarily interested in this third and final decision.

The combination operation binds individual sampling distributions into a larger encompassing sampling model.
A combination operator, denoted $h$ and parametrised by some model weights $\boldsymbol{w}$, maps a sequence of probability densities into a single probability  density:
\begin{equation}
    h (\pi_1(y), \dots, \pi_K(y) \mid \boldsymbol{w} ) = \pi_*(y), 
\end{equation}
subject to $\pi_*(y) \geq 0,$ for all $y$, and $\int_{\mathcal{Y}} \pi_*(y) \,\mathrm{d}y=1$, where we integrate with respect to the Lebesgue measure.
For example, a (linear) mixture can be represented by 
\begin{equation}
    h(\pi_1(y), \dots, \pi_K(y) \mid \boldsymbol{w} ) = \sum_{k=1}^K w_k \pi_k(y), 
\end{equation}
subject to $\Sigma_{k} w_k=1$.

In Bayesian statistics, the linear mixture is the \emph{de facto} combination operator to combine predictive distributions, and is found in  Bayesian model averaging \citep{raftery1997bayesian}, stacking \citep{yao2018using}, hierarchical stacking \citep{yao2021bayesian}, hypothesis testing \citep{kamary2018bayesian}, and mixture-of-experts \citep{jacobs1991adaptive,jordan1994hierarchical, yuksel2012twenty}.
Despite its mathematical convenience, the linear mixture has a few limitations:
\begin{enumerate*}[label=(\arabic*), ref=\arabic*, nosep]
\item linear combinations mean that one is restricted to a network of depth one when combining individual sampling models;
\item it only examines likelihoods through their evaluations at realised observations;
\item its linear nature typically results in a multimodal posterior predictive distribution, which comes with unnatural interpretation and poor interval coverage.
\end{enumerate*}
 
In this paper, we primarily consider combining Bayesian predictive distributions by geometric bridge (log-linear pooling, or \textit{locking}),
    \begin{equation}\label{eq_geoBri}
        h (\pi_1(\cdot), \dots, \pi_K(\cdot) \mid  \boldsymbol{w} ) \coloneqq  \frac{\prod_{k}  \pi_k^{w_k}(\cdot)}{\int_{\mathcal{Y}} \prod_{k}  \pi_k^{w_k} (y) \,d y},  
    \end{equation}
    where the weights lie in the $K$-dimensional simplex, $\boldsymbol{w} \in \simplex^K$.

Compared to the linear mixture, these new operators have appealing features: when individual sampling models are log-concave, so is their geometric (log-linear) bridge, hence preserving unimodality.
Moreover, in quantum superposition, when the phases $\alpha$ are uniformly distributed, we get back a linear mixture of densities.
Even when there is only one single model, depending on the phase, the superposition and geometric bridge can make the combined distributions spikier, or flatter -- approximately a power transformation, thereby automatically calibrating the prediction confidence.
In this sense, our locking approach automatically calibrates the prediction by self-interference.
Finally, unlike the linear mixture, the superposition and geometric bridge can create a middle mode, leading to more flexible predictions.

The remaining question is then how to optimise the weights $w_k$ such that the combined predictions best fit the data.
This is challenging because of the intractable normalising constant, and existing log-linear pooling techniques rely on some non-testable normal approximation \citep{poole2000inference,huang2005sampling, rufo2012log, Carvalho2023}.
In the next section, we provide a practical solution that incorporates the  Hyv\"arinen score \citep{hyvarinen2005estimation} and Bayesian posterior predictions. 
\section{Operator-oriented stacking}
\label{sec:operators}
In methods like stacking and mixture of experts, we need a scoring rule to evaluate the combined prediction \citep{gneiting2007strictly}.
The logarithmic scoring rule is \textit{de facto} the only continuous proper local scoring rule.
However, the log score does not easily apply to log-linear pooling and superposition: aside from trivial cases, the combined predictive densities contain an unknown normalisation constant in the denominator.  

To bypass computing this normalising constant, we use the Hyv\"arinen score \citep{hyvarinen2005estimation} to evaluate the unnormalised combined predictive density.
The latter has found application in producing posterior distributions with scoring rules \citep{Giummol__2018}, model selection under improper priors \citep{dawid_bayesian_2015, Shao2019}, and model selection of improper models \citep{jewson_general_2021}.
In general, given an unnormalised density $p$, how well it fits the observed data $y$ is quantified by 
$$
\mathcal{H}(y; p) = 2\Delta_y \log p(y) + \lVert\nabla_y \log p(y)\rVert^2.
 $$
The Hyv\"arinen score can then be interpreted as the $L_2$ norm of the difference between the score of the prediction and the true data generating process.

\subsection{Importance weighted estimate of the Hyv\"arinen score}
\label{sec:derivative_estimation}
Within Bayesian inference, posterior predictions are themselves a mixture of conditional sampling distributions.
For instance, the $k^\text{th}$ model's posterior parameter distribution given observed data $\D$ is denoted $p_k(\theta \mid \D)$, and its predictive density on future data  $\tilde y$ is $    \pi_k(\tilde y) = \int_{\boldsymbol{\Theta}} f_k(\tilde y \mid \theta) p_k(\theta \mid \D)\,\mathrm{d}\theta$, where we drop the dependence of $\pi_k$ on $\D$ for notational convenience.
To compute the Hyv\"arinen score of this posterior predictive distribution, we need the pointwise score functions $\Delta_y \log \pi_k(y)$ and $\nabla_y \log \pi_k(y)$.

We will typically use Markov chain Monte Carlo (MCMC) methods for individual model inference, such that we have $S$ simulation draws $\{\theta_{k}^{(s)}\}_{s = 1}^{S}$ from the $k^\text{th}$ model posterior $p_k(\theta \mid \D)$.
We compute both score functions by Monte Carlo sum, and provide a plug-in estimate of the score function by importance sampling.

In our case, we would like to sample from the first derivative of the posterior predictive distribution (our target), while only being able to actually sample from the log score of the model, $f_k(y\mid\theta_k)$, and the posterior distributions of the model parameters, $p_k(\theta \mid \D)$ (our proposal).
We thus consider an importance weighted estimator of the target 
\begin{IEEEeqnarray}{rl}
    \nabla_y \log \pi_k(y) \approx g_{k}(y) \coloneqq \frac{\sum_{s=1}^{S} \nabla_y f_k(y \mid \theta_{k}^{(s)}) }{\sum_{s=1}^{S} f(y \mid \theta_{k}^{(s)})}, \label{eq_score_est}
\end{IEEEeqnarray}
where $\theta_{k} = \left\{\theta_{k}^{(1)}, \ldots, \theta_{k}^{(S)} \right\}$ are draws from $p_k(\cdot \mid \D)$.
As such we find that the first derivative can be approximated as the Monte Carlo expectation taken with respect to
\begin{equation}
    h_{1k}(\theta^{(s)}) = \frac{\nabla_yf_k(y\mid\theta^{(s)})}{f_k(y\mid\theta^{(s)})^2}, \nonumber
\end{equation}
and with individual sample weights computed as
\begin{equation}
    \omega_k(\theta^{(s)}) = f_k(y\mid\theta^{(s)}) . \label{hyva-weight}
\end{equation}
Repeating similar steps for the second derivative, we achieve the approximation,
\begin{IEEEeqnarray}{rl}
     \Delta_y \log \pi_k(y) \;&= \frac{\pi_k(y)\pi_k^{\prime\prime}(y) - \pi_k^{\prime}(y)^2}{\pi_k(y)^2},\nonumber \\
      &\approx  \frac{\sum_{s=1}^{S} f_k^{\prime\prime}(y \mid \theta_{k}^{(s)}) }{\sum_{s=1}^{S} f_k(y \mid \theta_{k}^{(s)})} - g_{k}(y)^2. \nonumber
\end{IEEEeqnarray}
A complete derivation of these estimators is provided in Appendix~\ref{sec:appendix-is-hyva}, along with the function required to estimate the second derivative by \textit{via} Monte Carlo.
Note that the individual sample weights $\omega_k(\theta)$ are constant over all $h_{1k}(\theta)$.
There is no worry that the denominator and numerator are estimated using the same draws: we can view Equation~\ref{eq_score_est} as self-normalised importance sampling and thus the usual convergence theory \citep[e.g.,][]{mcbook} guarantees the consistency and asymptotic normality of our score function estimate.

One can easily compute the approximate leave-one-out cross-validated \citep[LOO-CV;][]{vehtari2017practical} Hyv\"arinen score by re-weighting the $s^{\text{th}}$ posterior sample by its log score weight, given in Equation~\ref{hyva-weight}.
\subsection{Score matched model stacking}
Our general model combination method revolves around optimising the Hyv\"arinen score of the combined posterior densities, and consists of four steps:
\begin{enumerate}[
    label={\textbf{Step \arabic*}:}, 
    ref=\arabic*, 
    itemindent=0pt, 
    leftmargin=!,
    labelindent=1em
]
\item Fit each model to the data and obtain $K$ predictive densities.
In practice, the posteriors  $p_{k}(\theta\mid\D)$ are represented by Monte Carlo draws, $\{\theta_{k}^{(s)}\}_{s=1}^{S}$, leading to the estimate $\hat{\pi}_k(\cdot) \coloneqq \frac{1}{S} \sum_{s=1}^{S}  f_k (\cdot \mid \theta_{k}^{(s)})$ of the predictive density.
\item Express the unnormalised predictive density via the combination operator.
For example, in locking we have 
\begin{equation}
    q(\cdot \mid  \boldsymbol{w}) := \prod_{k}  \pi_k^{w_k}(\cdot). \nonumber
\end{equation}
\item \label{step:3} Evaluate  $\nabla_y \log q(\cdot \mid  \boldsymbol{w})$ and $\Delta_y \log q(\cdot\mid \boldsymbol{w} )$ at all observed points, $y_i \in D$.
They come in closed form functions of $\nabla_y \pi_k (y_i \mid \theta_{k}^{(s)})$ and $\Delta_y \pi_k (y_i \mid \theta_{k}^{(s)})$.
  In locking: 
 \begin{IEEEeqnarray}{rl}
 q^\prime_{i} (\boldsymbol{w}) \;&\coloneqq  \nabla_y \log q(y_i \mid  \boldsymbol{w}), \nonumber \\
 \;&= \sum_{k=1}^K w_{k} \nabla_y \log \left( \pi_{k} (y_i)\right), \nonumber\\
 \;&\approx \sum_{k=1}^K w_{k}  \frac{\sum_{s=1}^S   \nabla_y   f_k (y_i \mid \theta_{k}^{(s)})   }{\sum_{s=1}^{S}   f_k (y_i \mid \theta_{k}^{(s)})  }, \nonumber\\
q^{\prime \prime}_i  (\boldsymbol{w}) \;&\coloneqq  \Delta_y \log q(y_i \mid  \boldsymbol{w}), \nonumber \\
\;&\approx \sum_{k=1}^K \frac{w_{k}  }{S} \sum_{s=1}^S \left(\Delta_y \log f_k (y_i \mid \theta_{k}^{(s)}) \right). \nonumber
 \end{IEEEeqnarray} 
These weights should ideally be computed using the LOO-CV Hyv\"arinen score, or independent test data.
\item Optimise the model weight vector $\boldsymbol{w}$ according to the constrained objective
    \begin{equation}
        \hat w_{\mathrm{opt}} = \argmin_{ \boldsymbol{w} \in\simplex^K} \left\{\sum_{i=1}^{n}\left({2q^{\prime \prime}_i  (\boldsymbol{w}) + \mid q^\prime_{i} (\boldsymbol{w})  \mid^2 }\right) - \log \mathrm{prior}(\boldsymbol{w}) \right\}.
        \label{eq_obj}
  \end{equation}
We use a non-informative Dirichlet prior (with concentration parameters all equal to 1.01) over the weight regularisation term $\boldsymbol{w}$. 
\end{enumerate}

\paragraph{Complexity.}

The key blessing of applying score matching to Bayesian predictions is that the Monte Carlo integral is linear in complexity, and is exchangeable with gradient operators.
Hence, all we need is to compute and store the gradient and hessian of the log likelihood (with respect to data, which is usually much lower dimension) at the sampled parameters once, that is $\nabla_y \pi_k (y_i \mid \theta_{k}^{(s)})$ and $\Delta_y \pi_k (y_i \mid \theta_{k}^{(s)})$.
In particular, the score functions have already been computed in gradient-based MCMC sampler, such as in dynamic Hamiltonian Monte Carlo \citep{hoffman2014no} and hence are nearly free.
The summation in the objective function (Equation~\ref{eq_obj}) contains $\bigoh{nKS}$ gradient evaluations in total, which can be done in parallel. 
 
\begin{figure}
    \centering
    \includegraphics[scale=0.5]{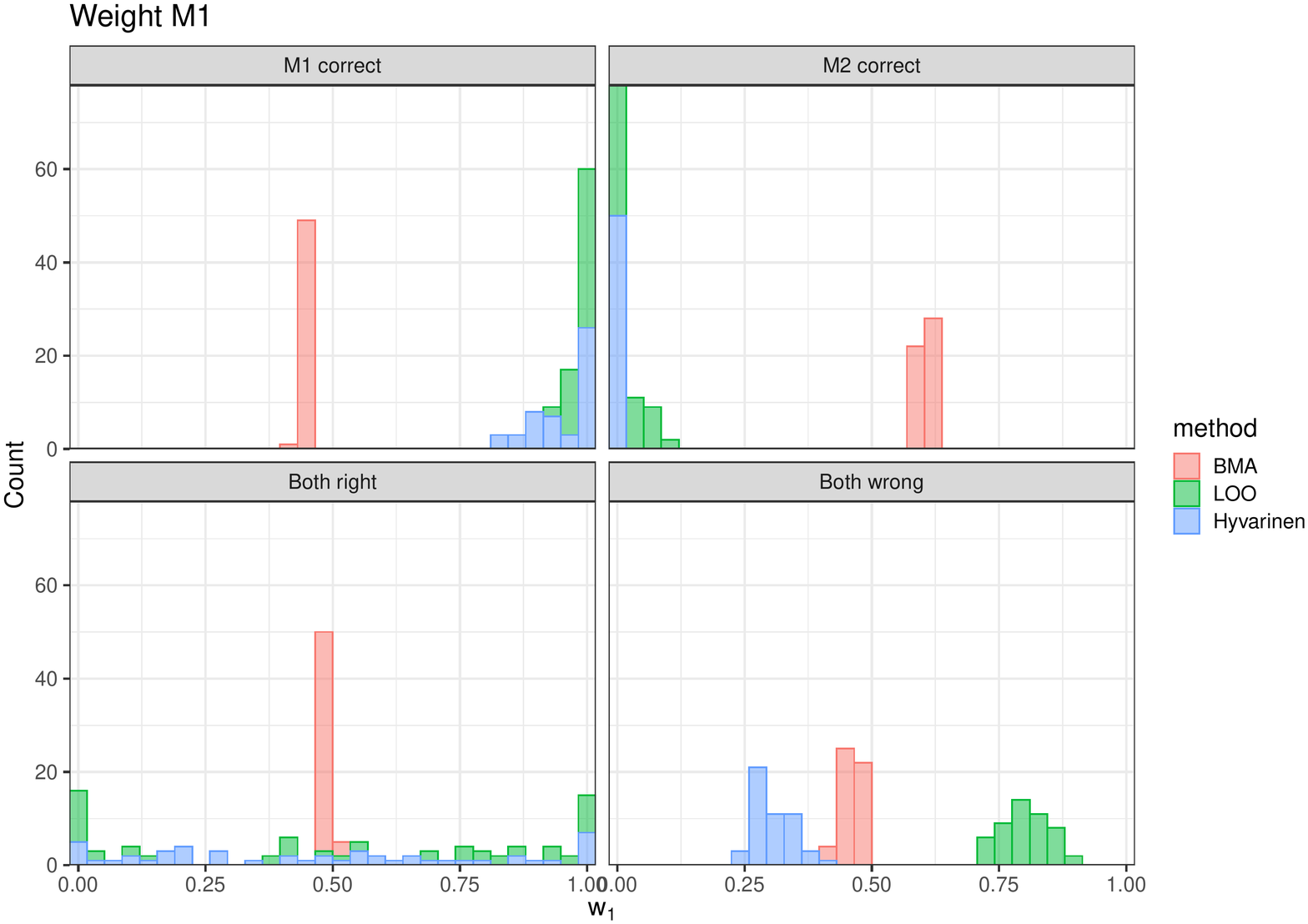}
    \caption{\small \textbf{Estimated Model 1 weights in repeated sampling under various scenarios}.
    We show the optimised value of the weight for model 1 ($w_1$) according to Bayesian model averaging (BMA), leave-one-out (LOO) stacking and  Hyv\"arinen model stacking.
    }    \label{fig:weightsNNN}
\end{figure}
\section{Motivating example with non-nested models}
\begin{figure} 
    \centering
    \includegraphics[scale=0.5]{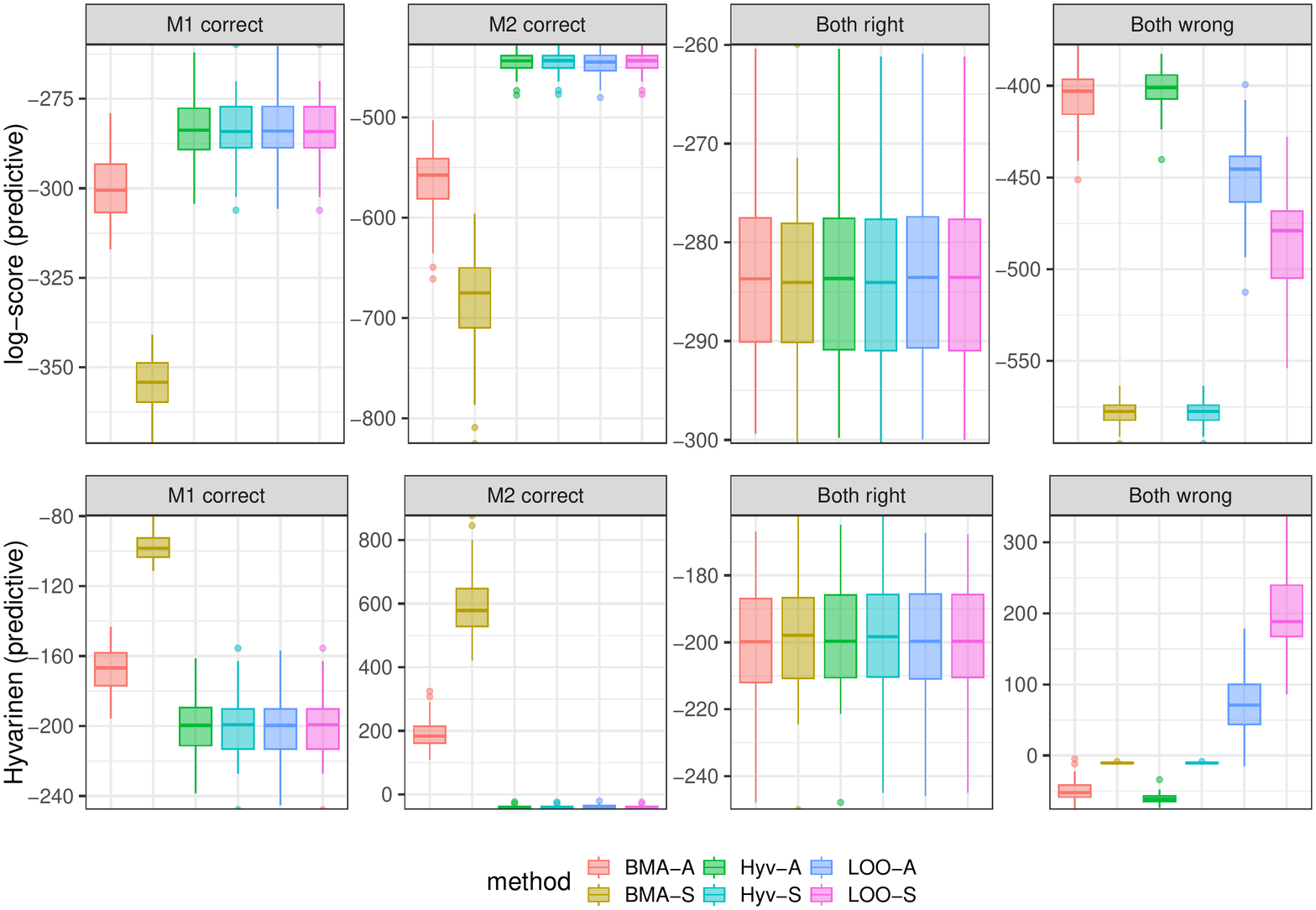}
    \caption{\small \textbf{Log predictive and Hyv\"arinen scores}.
    For each method, we show the overall predictive log score, $\sum_{j=1}^{N_{\textrm{pred}}}\log \pi_\ast(y_j)$. Our method (log-linear Hyv\"arinen model stacking, green) achieves higher log scores in all four scenarios.}
    \label{fig:logscores}
\end{figure}
To illustrate the flexibility of our new approach, we compare our proposed locking procedure to other state-of-the-art model averaging and selection tools. Consider two belief models adapted from \citet{Shao2019}:
\begin{IEEEeqnarray}{lrl}
    \mathcal{M}_1:& \quad Y_i \;&\sim \operatorname{normal}(\theta_1, 1),\nonumber \\
    & \theta_1 \;&\sim \operatorname{normal}(0, v_0); \nonumber \\
    \mathcal{M}_2:& \quad  Y_i \;&\sim \operatorname{normal}(0, \theta_2), \nonumber \\
    & \theta_2 \;&\sim \operatorname{inverse-}\chi^2(\nu_0, \tau_0). \nonumber
\end{IEEEeqnarray}
Following~\citet{Shao2019}, we take $v_0=10$, $\nu_0=0.1$ and $\tau_0=1$, and simulate $N_{\text{train}}$ data points from a true data generating process (a Gaussian distribution with mean $\mu^\star$ and variance $v_\star$) and generate $N_{\text{test}}$ independent test samples. Consider four scenarios: 
\begin{enumerate*}[label=(\arabic*), ref=\arabic*, nosep]
    \item $\mu^\star=1$ and $v^\star=1$ meaning that $\mathcal{M}_1$ is correctly specified but  $\mathcal{M}_2$ is not;
    \item $\mu^\star=0$ and $v^\star=5$ meaning that $\mathcal{M}_2$ is correctly specified but  $\mathcal{M}_2$ is not; 
    \item $\mu^\star=4$ and $v^\star=3$, a situation in which neither model is correctly specified;
    \item $\mu^\star=0$ and $v^\star=1$, in which both are correctly specified.
\end{enumerate*}
We ran $M=100$ replications of each scenario, with $N_{\text{train}}=200$ and $N_{\text{test}} = 50$.

We compare six methods in total:
\begin{enumerate*}[label=(\arabic*), ref=\arabic*, nosep]
    \item model selection using marginal likelihood;
    \item Bayesian model averaging;
    \item model selection using LOO-CV expected log-predictive density (elpd) \citep{vehtari2017practical};
    \item Bayesian stacking \citep{yao2018using};
    \item model selection using Hyv\"arinen score \citep{Shao2019};
    \item \emph{locking} (our proposed method).
\end{enumerate*}
We evaluate predictive performance of the learned combined model, where weights are computed using the training data. To make the comparison fair, we pick a metric that we do not directly optimise over: the log predictive density on test data. As shown in Figure~\ref{fig:logscores}, our new locking method is among the best-performing procedures in terms of the log score throughout all regimes.

We also demonstrate our proposed importance sampling routine from Section~\ref{sec:importance_sampling} on this case example. In Figure~\ref{fig:sampling} we visualise the posterior predictive draws from both models in all four cases, along with the posterior predictive of the locked model using importance sampling compared to the true data-generating process. Specifically, we find that we importance sampling is an efficient and accurate method of sampling from the locked posterior predictive, and further that this coincides with the true data-generating process, even the case where neither model is correctly specified.
\begin{figure} 
    \centering
    \includegraphics[width=0.8\textwidth]{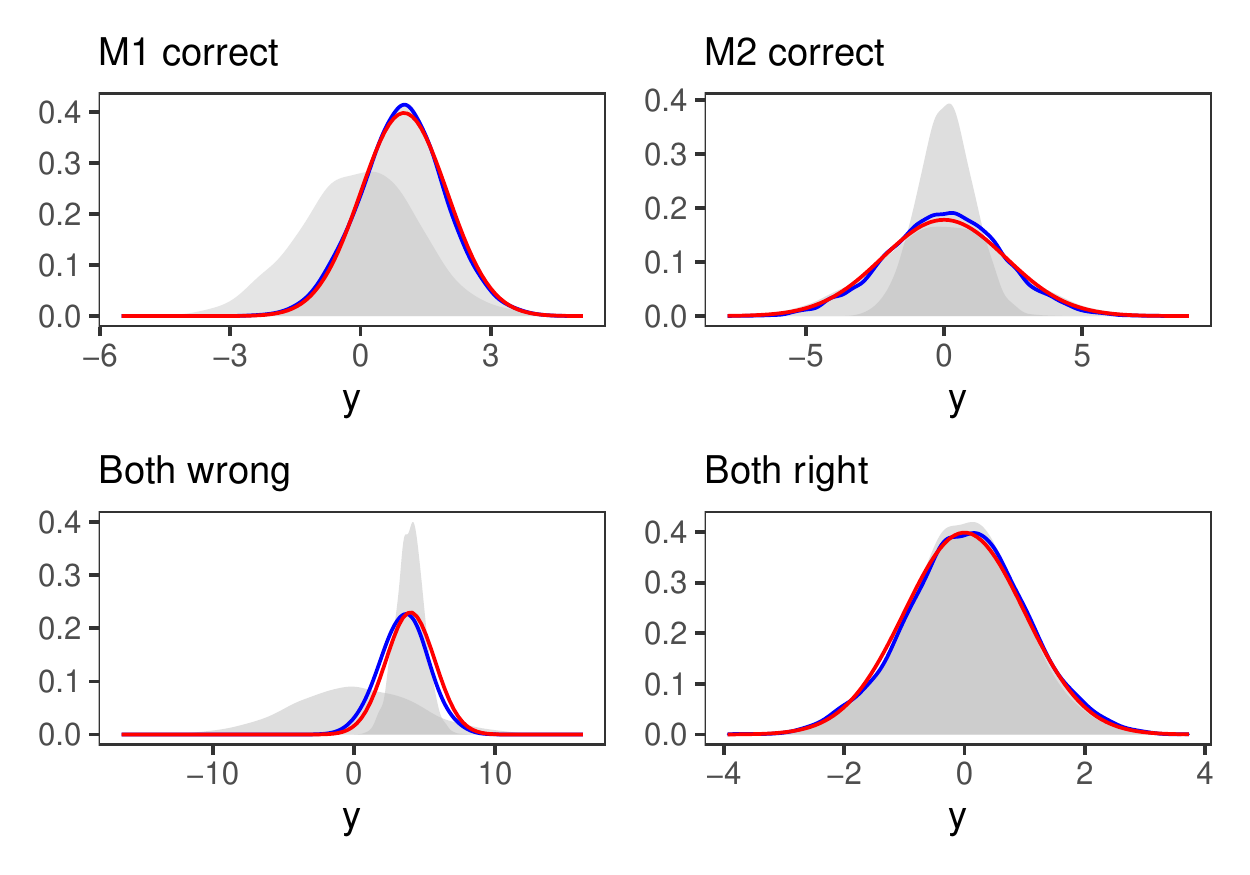}
    \caption{\small \textbf{Sampling from the locked posterior predictive}. The importance-sampled posterior predictive density of the locked model (in \textcolor{blue}{blue}) and the theoretical true data-generating process (in \textcolor{red}{red}) compared to the underlying constituent models (in \textcolor{gray}{grey}). We find that in this simple case, we can efficiently and accurately achieve the true predictive density with importance sampling -- even when neither of the models is correctly specified.}
    \label{fig:sampling}
\end{figure}

\begin{figure} 
    \centering
    \includegraphics[width=0.7\textwidth]{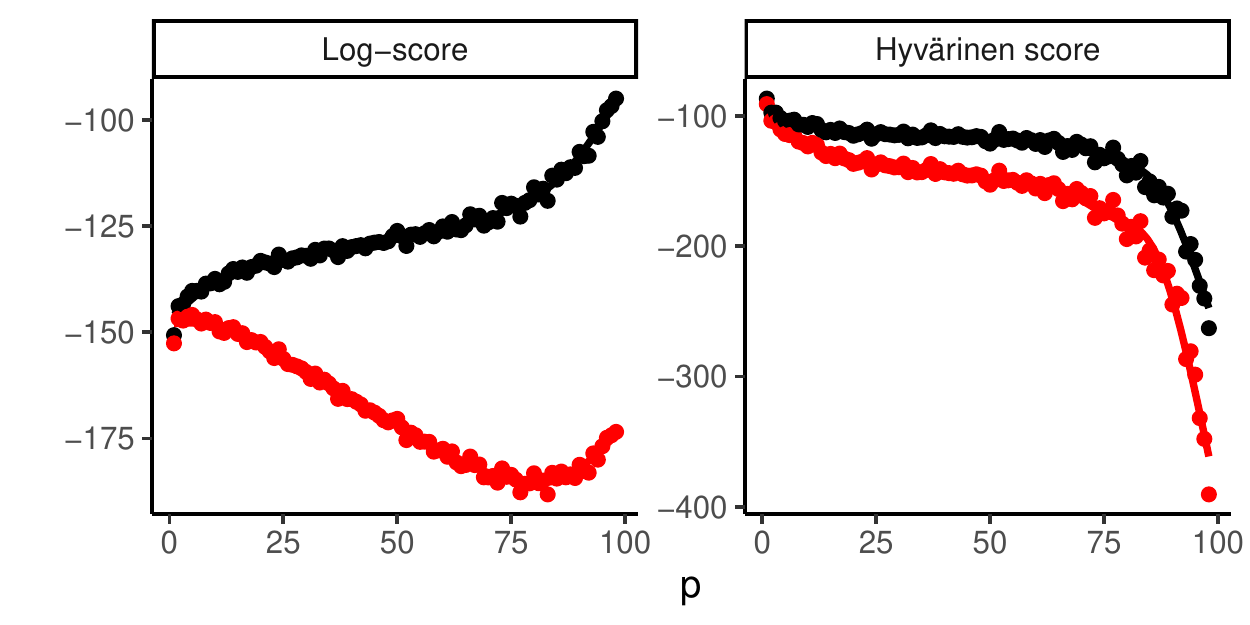}
    \caption{\small Simulated linear regression example. A comparison of mean in-sample (smoothed black line) and LOO-CV  (smoothed \textcolor{red}{red} line) Hyvärinen score and the log score over $100$ iterations for each regression dimension $p = 1,\dotsc,100$. We find that while in-sample and LOO-CV elpd scores diverge as $p$ grows, the in-sample Hyv\"arinen score remains close to its LOO value. Since locking operates on a different loss function than that which the models were trained on, the model averaging step is less likely to over-fit the data.}
    \label{fig:overfitting}
\end{figure}

\section{Discussion}
%
\subsection{Relative over-fitting of log-predictive stacking compared to score-matched stacking}
In addition to the ability of training an unnormalized model,  using the Hyv\"arinen score in the model averaging context provides extra immunity to overfitting: since the individual models are not trained the Hyv\"arinen score, it is likely that reusing the same training data to compute the Hyv\"arinen score as was used to train the individual models will not inject large amounts of bias to the inference of stacking weights $w_k$, analogously to Goodhart's Law \citep{goodhart_problems_1984}.\footnote{In a word, ``any observed statistical regularity will tend to collapse once pressure is placed upon it for control purposes''.} In contrast, a Bayesian model is guaranteed to over-fit the log score:  the in-sample log score is
$
    \sum_{i=1}^n \log \left( \frac{1}{S} \sum_{s=1}^S p(y_i \mid \theta_s)\right),
$  
while its LOO estimator is 
$ 
    \sum_{i=1}^n \log \left( \frac{S}{ \sum_{s = 1}^S p^{-1}(y_i \mid \theta_{s})} \right).
$ 
Since the harmonic mean
is always less than the arithmetic mean,  
any Bayesian model is guaranteed to have a lower leave-one-out log score point-wise. 
%

The immunity of the Hyv\"arinen score to overfitting  is empirically seen in the following simulated example: we generate $n = 100$ data points according to an underlying linear model with a fixed low signal-to-noise ratio. We then fit a linear regression model with different numbers of covariates $p$, using priors that are liable to over-fit and compute the in-sample and LOO-CV log-predictive density and Hyv\"arinen scores, shown in Figure~\ref{fig:overfitting}.\footnote{In particular, we impose independent wide Gaussian priors over the regression coefficients, and Student-$t$ priors over the intercept and variance parameter.} We find that as $p$ grows and the model over-fits the data in the log score, and the divergence between in-sample log score and LOO log score grows. The LOO Hyv\"arinen score however remains closer the in-sample Hyv\"arinen score as $p$ grows, suggesting that it does not over-fit the data as severely.

\subsection{Alternative combination operators}
\begin{figure} 
    \centering
    \includegraphics[width=\linewidth]{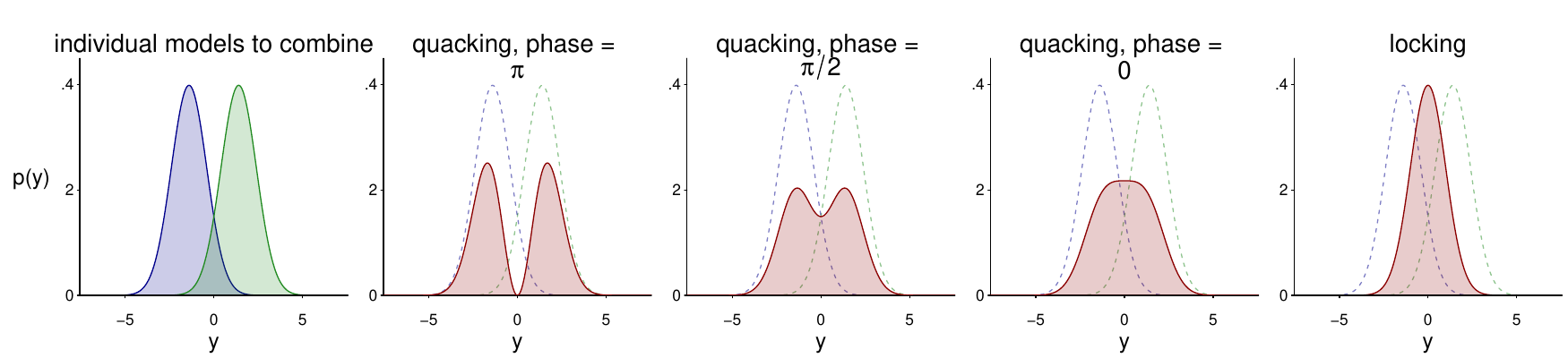}
    \caption{\small When combining two probabilistic predictions (in the left-most panel), quacking combines them via superposition and locking combines them by geometric bridge (log-linear pooling). The phase parameter $\alpha$ present in quacking dictates the degree of unimodality to enforce in the combined posterior, while this is done implicitly in locking.}
    \label{fig:demo}
\end{figure}
Alongside locking, we might consider combining Bayesian model predictions by quantum superposition (which we call \textit{quacking}), 
\begin{equation}\label{eq_super}
    h(\pi_1(\cdot), \dots, \pi_K(\cdot) \mid \boldsymbol{w},\alpha) \coloneqq \frac{ \big| \sum_k \sqrt{w_k}  \sqrt{\pi_k(\cdot)} e^{i\alpha_k}  \big|^2}{\int_{\mathcal{Y}} \big| \sum_k \sqrt{w_k}  \sqrt{\pi_k(y^)} e^{i\alpha_k}  \big|^2\,\mathrm{d} y},
\end{equation}
for $\boldsymbol{w} \in \simplex^K$ and $\alpha \in[0, 2\pi)^{K}$. In practice, we approximate this with the hybrid form
\begin{equation}
h (\pi_1(\cdot), \dots, \pi_K(\cdot) \mid \boldsymbol{w}, \boldsymbol{\beta}) \coloneqq  
\frac{\left(\sum_k \beta_k \pi_k(\cdot)  \right)^{w_0}\prod_{k}  \pi_k^{w_k}(\cdot)}{\int_{\mathcal{Y}} \left(\sum_k \beta_k \pi_k(y) \right)^{w_0} \prod_{k}  \pi_k^{w_k} (y) \,\mathrm{d}y},
\end{equation}
where now $\boldsymbol{\beta} \in \simplex^k, ~\boldsymbol{w} \in \R^{k + 1}$.
This combination regime induces different idiosyncrasies to locking, visualised in Figure~\ref{fig:demo}. The phase parameter $\alpha$ in quacking dictates the degree of unimodality enforced in the combined posterior, with $\alpha = 0$ producing the most unimodal posterior, and $\alpha = \pi$ preserving two distinct model. 
The quacking derivatives require to compute the Hyv\"arinen score also come in closed form expression (functions of weights $\boldsymbol{w}$ and phase $\alpha$) and are presented in Appendix~\ref{sec:appendix-quack-Hyva}. 

%
 %
Locking represents a step towards more exotic Bayesian model combinations, capable of alleviating some of the drawbacks of linear model mixtures. We have found that across different model regimes, score-matched locking is able to produce stable predictions, and by investigating the posterior model weights is able to identify the true model under the assumption that it exists. In the case where the true model is not included in the mixture, then achieving a low Hyv\"arinen score has some intuition as was discussed by \citet{jewson_general_2021}. Namely, since our mixture is unnormalised, the predictive point estimate can not be interpreted due to the lack of normalising constant. Instead, following \citet{ehm_local_2010} we find that local proper scoring rules (the family of which the Hyv\"arinen score is a member) can be decomposed into an accuracy and precision term \citep{jewson_general_2021}. As such, achieving a low Hyv\"arinen score can be interpreted to mean that the \textit{relative} predictions of the model and the true data-generating process are similar (high accuracy, low first derivative), and further that this accuracy is stable (high precision, low second derivative).
\subsection{Limitations}
Despite the extension to stacking this paper makes with general combination operators such as log-linear pooling through the Hyv\"arinen score, our approach has  limitations.  First, the   Hyv\"arinen score only applies to continuous outcomes.  Second, when the log score is accessible, the Hyv\"arinen score only matches the gradient and is therefore less efficient. 
Third, unlike a linear mixture, after we run log-linear stacking, it is typically non-trivial to sample from the outcome model  $ y\mid w \propto  {\prod_{k}  \pi_k^{w_k}(y)}$.   Because we are typically working with the situation with a low outcome dimension, we find simple importance sampling represent the locked posterior well in our experiments (Appendix \ref{sec:importance_sampling}).  A more efficient sampling algorithm remains an open problem and we leave it  for future investigations.

\bibliographystyle{apalike}
\bibliography{main}
\subsection*{Acknowledgments}
Y.M. acknowledges the computational resources provided by the Aalto Science-IT project.
\newpage
\begin{appendices}
\section{Importance weighted estimate of the Hyv\"arinen score}\label{sec:appendix-is-hyva}
We consider an importance weighted estimator of the first derivative of the log predictive of a given model $k$ with respect to the data $y$,
\begin{IEEEeqnarray}{rcl}
    \nabla_y \log\pi_{k}(y) \,&=& \frac{\pi_k^\prime(y)}{\pi_k(y)} \nonumber, \\
    &=& \frac{\int_\Theta \nabla_yf_k(y\mid\theta)p_k(\theta \mid \D)\,\mathrm{d}\theta}{\int_\Theta f_k(y\mid\theta)p_k(\theta \mid \D)\,\mathrm{d}\theta},\nonumber \\
    &=& \int_{\Theta}\frac{\frac{\partial}{\partial y}f_k(y\mid\theta)}{f_k(y\mid\theta)^2}\,\cdot\,\frac{f_k(y\mid\theta)^2p_k(\theta \mid \D)}{\int_{\Theta} f_k(y\mid\theta)p_k(\theta \mid \D) \,\mathrm{d}\theta}\,\mathrm{d}\theta,\nonumber \\
    &\approx\,& \sum_{s = i}^\mathcal{S} \frac{\nabla_yf_k(y\mid\theta^{(s)})}{f_k(y\mid\theta^{(s)})^2} \,\cdot\,\frac{f_k(y\mid\theta^{(s)})^2}{ \Sigma_{s-1}^{S}f_k(y\mid\theta^{(s)})}, \nonumber \\
    &=& \frac{\sum_{s=1}^{S}\nabla_yf_k(y\mid\theta^{(s)})f_k(y\mid\theta_k^{(s)})}{\sum_{s=1}^{S}f_k(y\mid\theta_k^{(s)})^2}, \nonumber \\
    &=& \frac{\sum_{s=1}^{S}\nabla_yf_k(y\mid\theta_k^{(s)})}{\sum_{s=1}^{S}f_k(y\mid\theta_k^{(s)})} \eqqcolon g_{k}(y),\label{eq:nabla-is}
\end{IEEEeqnarray}
where once again the samples $\boldsymbol{\theta_k} = \{ \theta_k^{(1)}), \ldots, \theta_k^{(S)})\}$ are assumed to be (approximately) drawn from the posterior $p_k(\cdot \mid \D)$.
As such we find that the importance weighted expectation is taken with respect to the function
\begin{equation}
    h_1(\theta^{(s)}) = \frac{\nabla_yf_k(y\mid\theta^{(s)})}{f_k(y\mid\theta^{(s)})^2}, \nonumber
\end{equation}
and that the individual sample weights are
\begin{equation}
    \omega(\theta^{(s)}) = f_k(y\mid\theta^{(s)}) . \nonumber
\end{equation}
Repeating similar steps for the second derivative, we first note the connection to $g_{k}(y)$ in that
\begin{equation}
    \Delta_y \log \pi_{k}(y) = \frac{\pi_k^{\prime\prime}(y)\pi_k(y) - \pi_k^\prime(y)^2}{\pi_k(y)^2} = \frac{\pi_k^{\prime\prime}(y)}{\pi_k(y)} - \left(\nabla_y \log\pi_{k}(y)\right)^2. \nonumber
\end{equation}
As such, we begin by investigating,
\begin{IEEEeqnarray}{rcl}
    \frac{\pi_k^{\prime\prime}(y)}{\pi_k(y)}\,&=& \frac{\int_\Theta \Delta_yf_k(y\mid\theta)p_k(\theta \mid \D)\,\mathrm{d}\theta}{\int_\Theta f_k(y\mid\theta)p_k(\theta\mid \D)\,\mathrm{d}\theta}\nonumber, \\
    &=& \int_{\Theta}\frac{\Delta_yf(y\mid\theta)}{f_k(y\mid\theta)^2}\,\cdot\,\frac{f_k(y\mid\theta)^2p_k(\theta\mid \D)}{\int_{\Theta} f_k(y\mid\theta)p_k(\theta\mid \D) \,\mathrm{d}\theta}\,\mathrm{d}\theta,\nonumber\\
    &\approx\,& \sum_{s = i}^\mathcal{S} \frac{\Delta_yf_k(y\mid\theta^{(s)})}{f_k(y\mid\theta^{(s)})^2} \,\cdot\,\frac{f_k(y\mid\theta^{(s)})^2}{ \sum_{s-1}^{S}f_k(y\mid\theta^{(s)})},\nonumber \\
    &=& \frac{\sum_{s=1}^{S}\Delta_yf(y\mid\theta^{(s)})}{\sum_{s=1}^{S}f(y\mid\theta^{(s)})^2}, \nonumber \\
    &=& \frac{\sum_{s=1}^{S}\Delta_y f_k(y\mid\theta^{(s)})}{\sum_{s=1}^{S}f(y\mid\theta^{(s)})}.\nonumber
\end{IEEEeqnarray}
From this, we produce the importance weighted of the second derivative as
\begin{equation}
    \Delta_y \log \pi_{k}(y) \approx \frac{\sum_{s=1}^{S}\Delta_yf_k(y\mid\theta^{(s)})}{\sum_{s=1}^{S}f_k(y\mid\theta^{(s)})} - g_k(y)^2, \label{eq:delta-is}
\end{equation}
so that the importance weight for both derivative terms is the same $\omega(\theta)$ as before, but that expectation is now taken with respect to
\begin{equation}
    h_2(\theta^{(s)}) = \frac{\Delta_yf_k(y\mid\theta^{(s)})f_k(y\mid\theta^{(s)}) - \left(\nabla_yf(y\mid\theta^{(s)})\right)^2}{f_k(y\mid\theta^{(s)})^3}. \nonumber
\end{equation}

Recall that the Hyv\"arinen score at some observation $y$ for the posterior predictive distribution $\pi_k$ is defined as
\begin{equation}
    \Hyv(y; \pi_k) = 2\Delta_y \log \pi_k(y) + \left\lVert \nabla_y \log \pi_k(y)\right\rVert^2. \nonumber
\end{equation}
Combining now our importance sample approximations of the first and second derivative from Equations~\ref{eq:nabla-is} and~\ref{eq:delta-is}, we can produce an importance sampled version of the complete Hyv\"arinen score as 
\begin{IEEEeqnarray}{rl}
    \Hyv(y; \pi_k) \,&\approx 2\left(\frac{\sum_{s=1}^{S}\Delta_yf_k(y\mid\theta^{(s)})}{\sum_{s=1}^{S}f_k(y\mid\theta^{(s)})}\right) - g_k(y)^2. \nonumber
\end{IEEEeqnarray}
Note that this too can be rewritten in terms of the individual importance weights of Equation~\ref{hyva-weight}. As such, we are always able to diagnose very poor importance sampling by taking the logarithm of the model's log score and computing the Pareto shape parameter $\hat k$ of the tail of their sample weights. This diagnostic is inherent to Pareto smoothed importance sampling \citep{psis}. In the case where many of these Pareto $\hat k$ values are too high ($>0.7$, say), we can understand the variance of our importance-weighted estimator to be near infinite and as a result the central limit theorem required for estimator consistency is no longer guaranteed to hold.
\section{The Hyv\"arinen score for quacked posteriors}
\label{sec:appendix-quack-Hyva}
We begin by denoting
\begin{equation}
    q(\cdot\mid w) = \left(\sum_{k=1}^K\beta_k\pi_k(\cdot)\right)^{w_0}\prod_{k=1}^K\pi_{k}^{w_k}(\cdot),\nonumber
\end{equation}
wherein $\boldsymbol{\beta} \in \mathcal{S}^k, ~\boldsymbol{w} \in \mathbb{R}^k$.
Then,
\begin{IEEEeqnarray}{rl}
    \nabla_y \log q(\cdot\mid w) \;&= \nabla_y\log\left\{\left(\sum_{k=1}^K\beta_k\pi_k(\cdot)\right)^{w_0}\prod_{k=1}^K\pi_{k}^{w_k}(\cdot) \right\} \nonumber\\
    &= \underbrace{\nabla_y w_0\log\left(\sum_{k=1}^K\beta_k\pi_k(\cdot)\right)}_{(\dagger)} + \underbrace{\nabla_y\log\left(\prod_{k=1}^K\pi_k^{w_k}(\cdot)\right)}_{(\ddagger)}.\nonumber
\end{IEEEeqnarray}
Beginning with $\ddagger$, we have previously shown that
\begin{IEEEeqnarray}{rl}
    \nabla_y\log\left(\prod_{k=1}^K\pi_k^{w_k}(\cdot)\right)\;&= \sum_{k=1}^K w_{k} \nabla_y \log \left( \pi_{k} (\cdot)\right), \nonumber\\
 \;&\approx \sum_{k=1}^K w_{k}  \frac{\sum_{s=1}^S   \nabla_y   f_k (\cdot \mid \theta_{k}^{(s)})   }{\sum_{s=1}^{S}   f_k (\cdot \mid \theta_{k}^{(s)})  }. \nonumber
\end{IEEEeqnarray}
Moving now to $\dagger$, note that
\begin{equation}
    \nabla_y w_0\log\left(\sum_{k=1}^K\beta_k\pi_k(\cdot)\right) = w_0\frac{\sum_{k=1}^K\beta_k\nabla_y\pi_k(\cdot)}{\sum_{k=1}^K\beta_k\pi_k(\cdot)},\nonumber
\end{equation}
wherein we can substitute $\pi_k(\cdot)$ with its Monte Carlo estimate $S^{-1} \sum_{s=1}^{S}   f_k (\cdot \mid \theta_{k}^{(s)})$ to achieve
\begin{IEEEeqnarray}{rl}
    \nabla_y w_0\log\left(\sum_{k=1}^K\beta_k\pi_k(\cdot)\right) \;&= w_0\frac{\sum_{k=1}^K\beta_k\nabla_y\sum_{s=1}^{S}   f_k (\cdot \mid \theta_{k}^{(s)})}{\sum_{k=1}^K\beta_k \sum_{s=1}^{S}   f_k (\cdot \mid \theta_{k}^{(s)})},\nonumber\\
      &= w_0\frac{\sum_{k=1}^K\beta_k\sum_{s=1}^{S}  \nabla_y f_k (\cdot \mid \theta_{k}^{(s)})}{\sum_{k=1}^K\beta_k \sum_{s=1}^{S}   f_k (\cdot \mid \theta_{k}^{(s)})}.\nonumber
\end{IEEEeqnarray}
Combining these we find 
\begin{equation}
    \nabla_y \log q(\cdot\mid w) \approx w_0\frac{\sum_{k=1}^K\beta_k\sum_{s=1}^{S}  \nabla_y f_k (\cdot \mid \theta_{k}^{(s)})}{\sum_{k=1}^K\beta_k \sum_{s=1}^{S}   f_k (\cdot \mid \theta_{k}^{(s)})} + \sum_{k=1}^K w_{k}  \frac{\sum_{s=1}^S   \nabla_y   f_k (\cdot \mid \theta_{k}^{(s)})   }{\sum_{s=1}^{S}   f_k (\cdot \mid \theta_{k}^{(s)})  }. \label{eq:quacking-nabla}
\end{equation}

We can then move to the second-order derivative needed for the Hyv\"arinen score:
\begin{IEEEeqnarray}{rl}
    \Delta_y \log q(\cdot\mid w) \;&= \Delta_y\log\left\{\left(\sum_{k=1}^K\beta_k\pi_k(\cdot)\right)^{w_0}\prod_{k=1}^K\pi_{k}^{w_k}(\cdot) \right\} \nonumber\\
    &= \underbrace{\Delta_y w_0\log\left(\sum_{k=1}^K\beta_k\pi_k(\cdot)\right)}_{(\star)} + \underbrace{\Delta_y\log\left(\prod_{k=1}^K\pi_k^{w_k}(\cdot)\right)}_{(\star\star)}.\nonumber
\end{IEEEeqnarray}
Here, we have from before that
\begin{IEEEeqnarray}{rl}
    \Delta_y\log\left(\prod_{k=1}^K\pi_k^{w_k}(\cdot)\right)\;&= \sum_{k=1}^K w_k\Delta_y\log (\pi_{k}(\cdot)) \nonumber\\
    &= \sum_{k=1}^Kw_k\frac{\Delta_y\pi_{k}(\cdot)\pi_{k}(\cdot) - \nabla_y\pi_{k}(\cdot)^2}{\pi_{k}(\cdot)^2} \nonumber\\
    &\approx \sum_{k=1}^K w_k\left(\frac{\sum_{s=1}^S\Delta_yf_k(\cdot\mid\theta_k^{(s)})}{\sum_{s=1}^S f_k(\cdot\mid\theta_k^{(s)})} - \left( \frac{\sum_{s=1}^S\nabla_yf_k(\cdot\mid\theta_k^{(s)})}{\sum_{s=1}^S f_k(\cdot\mid\theta_k^{(s)})}\right)^2\right), \label{eq:quacking-delta-i}
\end{IEEEeqnarray}
and similarly to before for that
\begin{IEEEeqnarray}{rl}
    \Delta_y w_0\log\left(\sum_{k=1}^K\beta_k\pi_k(\cdot)\right) \;&= w_0\frac{\sum_{k=1}^K\beta_k\Delta_y\pi_k(\cdot)}{\sum_{k=1}^K\beta_k\pi_k(\cdot)} \nonumber\\
    &\approx w_0\frac{\sum_{k=1}^K\beta_k\sum_{s=1}^{S}  \Delta_y f_k (\cdot \mid \theta_{k}^{(s)})}{\sum_{k=1}^K\beta_k \sum_{s=1}^{S}   f_k (\cdot \mid \theta_{k}^{(s)})}.\label{eq:quacking-delta-ii}
\end{IEEEeqnarray}
We combine Equations~\ref{eq:quacking-nabla},~\ref{eq:quacking-delta-i}, and~\ref{eq:quacking-delta-ii} to achieve the Hyv\"arinen score for quacked posteriors.
\section{Sampling from the locked predictive}\label{sec:importance_sampling}
Unlike a linear mixture, it is non-trivial to sample from a log-linear density $\prod \pi_k^{w_k}(\cdot)$
from existing sample draws. One quick approximation is to run importance sampling using the equally weighted proposal $1/K\sum_{k=1}^K \pi_k(\cdot)$.
 
Provided the modes of the constituent modes are relatively concentrated (that is the the models we are stacking have modes not too distant from one another), importance sampling should be able to represent the locked posterior well. Indeed, in general the mode of the locked posterior is bounded by the extremal modes of the component models when they are all unimodal.
 
\begin{proposition}[\textbf{The mode of the log-pooled density}]
\label{prop:logpool_mode}
Suppose each $\pi_k$ is unimodal and let $\boldsymbol{m} = \{m_1, \ldots, m_K \}$ be the modes of each density.
Further, let $m_a$ and $m_b$, $a, b \in [1, \ldots, K]$ be the smallest and largest such modes respectively.
Then the log-pooled density, $p(\cdot\mid\boldsymbol{w}) = c(\boldsymbol{w}) \prod_{k=1}^K \pi_k(\cdot)^{w_k}$ is unimodal with mode $m^\star \in [m_a, m_b]$.
\end{proposition}

\begin{proof}
The unimodality of the locked posterior follows from Theorem 2.2 by \cite{Carvalho2023} under the assumption that all consistuent models are unimodal.
For boundedness of the mode in location, let us first represent the densities by $\pi_k(\cdot) = \exp(\mu_k(\cdot))$ and get $p(\cdot\mid \boldsymbol{w}) = \exp\left(\sum_{k=1}^K w_k \mu_k(\cdot) \right)$ for $\boldsymbol{w}\in\simplex^K$.
Notice that by assumption each $\mu_k$ is monotonically non-decreasing on $(-\infty, m_k)$ and monotonically non-increasing on $[m_k, \infty)$.
The proof will proceed by contradiction.
Suppose $m^\star < m_a$, which implies
\begin{IEEEeqnarray*}{lrl}
    &p(m^\star\mid \boldsymbol{w}) &\;> p(m_a\mid \boldsymbol{w}),\\
    \implies &\;\exp\left(\sum_{k=1}^K w_k \mu_k(m^\star) \right) &\;> \exp\left(\sum_{k=1}^K w_k \mu_k(m_a) \right),\\
    \implies &\;\sum_{k=1}^K w_k \mu_k(m^\star)  &\;> \sum_{k=1}^K w_k \mu_k(m_a).
\end{IEEEeqnarray*}
Re-arranging gives
\begin{IEEEeqnarray}{lrl}
\nonumber
 &\sum_{j \neq a}^K w_j \mu_j(m^\star) + w_a\mu_a(m^\star)  &\;> \sum_{j \neq a}^K w_j \mu_j(m_a) + w_a\mu_a(m_a),\\
 \label{eq:contrad}
 \implies &\; \sum_{j \neq a}^K w_j \left[ \mu_j(m^\star)-\mu_j(m_a)\right] &\;> w_a \left[\mu_a(m_a)-\mu_a(m^\star)\right] > 0.
 \end{IEEEeqnarray}
 However, since $\mu_j(m^\star) \leq \mu_j(m_a)$ for all $j \neq a$ by assumption, their (weighted) sum is negative, leading to a contradiction in Equation~\ref{eq:contrad}.
 The argument for $m^\star > m_b$ is analogous and completes the proof.
\end{proof}
\end{appendices}
\end{document}